\pdfoutput=1
\documentclass[letterpaper]{article}
\usepackage[totalwidth=480pt, totalheight=680pt]{geometry}
\usepackage{natbib}
\usepackage{macro}

\newtheoremstyle{named}{}{}{\itshape}{}{\bfseries}{.}{.5em}{\thmnote{#3}}
\theoremstyle{named}

\newtheorem*{namedlemma}{Lemma}

\usepackage{authblk}

\title{Membership Testing in Markov Equivalence Classes\\via Independence Query Oracles}
\author[1,2,*]{Jiaqi Zhang}
\author[2,*]{Kirankumar Shiragur}
\author[1,2]{Caroline Uhler}
\affil[1]{LIDS, Massachusetts Institute of Technology}
\affil[2]{Eric and Wendy Schmidt Center, Broad Institute of MIT and Harvard}
\affil[*]{Equal contributions. Reversed alphabetical order.}
\date{}

\begin{document}

\maketitle
\addtocontents{toc}{\protect\setcounter{tocdepth}{0}}
\begin{abstract}

Understanding causal relationships between variables is a fundamental problem with broad impacts in numerous scientific fields.
While extensive research have been dedicated to \emph{learning} causal graphs from data, its complementary concept of \emph{testing} causal relationships has remained largely unexplored.
In our work, we take the initiative to formally delve into the \emph{testing} aspect of causal discovery. 
While \emph{learning} involves the task of recovering the Markov equivalence class (MEC) of the underlying causal graph from observational data, our \emph{testing} counterpart addresses a critical question: \emph{Given a specific MEC, can we determine if the underlying causal graph belongs to it with observational data?}

We explore constraint-based testing methods by establishing bounds on the required number of conditional independence tests. Our bounds are in terms of the maximum degree ($d$), the size of the largest clique ($s$), and the size of the largest undirected clique ($s'$) of the given MEC.
In the worst case, we show a lower bound of $\exp(\Omega(s'))$ independence tests. 
We then give an algorithm that resolves the task with $\exp(O(s))$ independence tests. Notably, our lower and upper bounds coincide when considering moral MECs ($s=s'$).
Compared to \emph{learning}, where $\exp(\Theta(d))$ independence tests are deemed necessary and sufficient, our results show that \emph{testing} is a relatively more manageable task and requires exponentially less independence tests in graphs featuring high maximum degrees and small clique sizes.

\end{abstract}

\section{INTRODUCTION}

% \jj{I've change all the formating below to aistats requirement, including section/subsection/paragraph/reference titles. Use ``citep'' for regular-reference and ``citet'' for inline-reference.}

% \kk{Will add references in some time}

% The study of causal relationships between variables which are often represented using directed graphs has found applications in many disciplines including biology, epidemiology, economics, and social science \citep{friedman2000using,robins2000marginal, spirtes2000causation, pearl2003causality}. Due to their immense applications, inferring information about these causal relationships from the data has been quite a well studied problem. It is well known that from the observational (non experimental) data, it is only possible to learn this underlying causal graph only upto its Markov equivalence class\footnote{The Markov equivalence class includes the set of all directed graphs that satisfy the same set of conditional independence test constraints.} (MEC) and most of the prior work has been dedicated towards understanding the complexity of learning the MEC given the observational data. While learning has been substantially explored, a natural counterpart to it of testing remains largely unexplored. 

The study of causal relationships, often represented using directed acylic graphs, has found practical applications in multiple fields including biology, epidemiology, economics, and social sciences \citep{king2004functional,cho2016reconstructing,tian2016bayesian, sverchkov2017review,rotmensch2017learning,pingault2018using, de2019combining, reichenbach1956direction,woodward2005making,eberhardt2007interventions, hoover1990logic, friedman2000using,robins2000marginal, spirtes2000causation, pearl2003causality}. Due to its importance and widespread utility, the challenge of reasoning about causal connections using data has been the subject of substantial research. With observational (i.e., non-experimental) data, it is well-known that the underlying causal graph is generally only identifiable up to its Markov equivalence class
% \footnote{We provide formal definitions in Section~\ref{sec:prelim}. Intuitively, an MEC is a set of DAGs that represent distributions satisfying the same set of conditional independence constraints.} 
(MEC) \citep{verma1990equivalence}. Prior research have investigated the complexities of learning MECs from observational data (c.f.,~\citet{spirtes2000causation,chickering2002optimal,colombo2011learning,solus2021consistency}). While significant attention has been devoted to the exploration of learning, the complementary theme of testing specific aspects of the hidden causal graph remains under-explored.

Learning and testing problems are commonly studied in various fields, including information theory \citep{fisher1943relation, orlitsky2003always} and learning theory \citep{good1953population, goldreich1998property, rubinfeld1996robust, mcallester2000convergence}. The concept of testing becomes particularly relevant in scenarios with limited data, where traditional learning methods are no longer viable. As an example, a groundbreaking discovery by \citet{paninski2008coincidence} established that testing if a hidden distribution supported on $k$ elements is close to a uniform distribution can be accomplished using just $O(\sqrt{k})$ (sublinear) samples. In contrast, learning the distance to the uniform distribution necessitates $\Omega(k)$ (linear) samples, making testing a considerably easier problem. This revelation has prompted researchers \citep{indyk2012approximating, batu2000testing, chan2014optimal, valiant2017automatic, batu2001testing} to investigate whether testing is generally easier than learning, a trend that has significant impact when data is scarce. 

In light of these findings, our work focuses on understanding the complexity of learning and testing problems in the context of causal discovery. As learning is concerned with the recovery of MECs from observational data, we consider the natural testing counterpart: 
\begin{center}
    \emph{Civen a specific MEC and
observational data from a causal graph,
can we determine if the data-generating causal
graph belongs to the given MEC?}
\end{center}
This inquiry opens up a novel and important avenue in the field of causal inference, focusing on the validation and assessment of pre-defined causal relationships within a given equivalence class. For example, such an equivalence class could be provided by a domain expert \citep{choo2023active, scheines1998tetrad,de2011efficient,flores2011incorporating,li2018bayesian} or a hypothesis generated by AI \citep{long2023causal,vashishtha2023causal}; and the testing problem aims to confirm the expert's guidance with minimal effort and data. In this context, we explore constraint-based methods and investigate the complexity of conditional independence tests, assuming standard Markov and faithfulness assumptions \citep{lauritzen1996graphical,spirtes2000causation}. 

%andrews2020completeness,fang2020ida

We demonstrate that, in the worst-case scenario, any constraint-based method still requires a minimum of $\exp(\Omega(s))$ number of conditional independence tests to solve the testing problem, where $s$ signifies the size of the maximum undirected clique in the given MEC. Complementing this result, we also introduce an algorithm that resolves the testing problem using at most $\exp(O(s) + O(\log n))$ tests. Our lower and upper bounds coincide\footnote{Ignoring the polynomial terms in $n$.} asymptotically in the exponents.

% It is a well-established fact that, when it comes to the \emph{learning} task, any constraint-based learning method necessitates a minimum of $\exp(\widetilde{\Omega}(d))$ conditional independence tests to recover the Markov Equivalence Class (MEC), where $d$ represents the maximum degree of the completed partially oriented DAG (CPDAG) representing the MEC. In contrast, the renowned PC algorithm accomplishes the learning task with roughly the same number of tests asymptotically.

Comparing our \emph{testing} results to the \emph{learning} problem, we remark here that most well-known constraint-based learning algorithms, including PC \citep{spirtes2000causation} and others \citep{verma1990equivalence, spirtes1989causality, spirtes2000causation}, in the worst-case, require an exponential number of conditional independence tests based on the maximum in-degree of the graph. As the maximum undirected clique size is no more than the maximum in-degree\footnote{Consider the most downstream node in the clique.}, it is evident that testing, although entailing an exponential number of tests, is still an easier task than its learning counterpart. Additionally, testing becomes significantly easier than learning on graphs featuring high in-degrees and small clique sizes.

\paragraph{Organization} 
In Section~\ref{sec:prelim}, we provide formal definitions of relevant concepts. In Section~\ref{sec:result}, we state our main results and provide an overview of the techniques used to derive them. We then unravel the lower bound result in Section~\ref{sec:lowerbound}. Our algorithm for testing and its analysis are presented in Section~\ref{sec:upperbound}. In Section~\ref{sec:polytope}, we provide a geometric interpretation of our results using the DAG associahedron. Finally, we conclude and discuss future works in Section~\ref{sec:discuss}.

% In Section~\ref{sec:prelim}, we provide formal definitions and useful results. In Section~\ref{sec:results}, we state all our main results. We provide our algorithm for Meek separator and its analysis in Section~\ref{sec:meek_sep}. We use the Meek separator subroutine to solve for subset search in Section~\ref{sec:sub_search} and causal matching in Section~\ref{sec:causal_state}. In Section~\ref{sec:experiments}, we demonstrate empirically our proposed algorithms on synthetic data.
\subsection{Related Works}

Learning causal relationships from observational data is a well-established field, with methods broadly falling into three main categories: constraint-based methods \citep{verma1990equivalence, spirtes1989causality, spirtes2000causation, kalisch2007estimating}, score-based methods \citep{chickering2002optimal, geiger2002parameter,brenner2013sparsityboost,solus2021consistency}, and other hybrid approaches \citep{schulte2010imap,alonso2013scaling, nandy2018high}. Score-based methods evaluate causal graphs (or MECs) by assigning scores that reflect their compatibility with the data. They then solve a combinatorial optimization problem to identify the graph with the best score. In contrast, constraint-based methods infer the causal structure by examining independence constraints imposed by the underlying causal graph on the data distribution. As one of the leading constraint-based algorithms, PC \citep{spirtes2000causation} starts with a complete undirected graph and systematically eliminates edges by performing conditional independence tests with increasing cardinality. The number of tests required for the PC algorithm to recover the true causal graph is roughly $\frac{n^2 (n-1)^{(d-1)}}{(d-1)!}$, where $n$ is the number of vertices and $d$ is the maximum in-degree.

The PC algorithm assumes causal sufficiency, i.e., no latent confounders. Assuming this, another prominent example is the IC algorithm \citep{verma1990equivalence}, whose complexity is bounded exponentially by the maximum clique size of the underlying Markov network. Note that as all parents of a node in the DAG are connected in its Markov network, this complexity is at least the exponent in maximum in-degree. To handle violations of causal sufficiency, \citet{spirtes2013causal} introduced FCI that invokes additional steps after PC. Subsequent work by \citet{claassen2013learning} presents FCI+, a modified version of FCI, that resolves the task in worst case $n^{O(d)}$ tests. In general, the learning problem is NP-hard \citep{chickering2004large}. Compared to these learning results, our algorithm establishes that testing can be solved in $n^{O(s)}$ tests, where $s$ is the maximum undirected clique size. As we will show in the next section, it always holds that $s\leq d$ in any DAG.
 
We note that, in other contexts,  learning and testing are well-studied problems. In recent years, there has been significant attention to understanding the time and sample complexities for both learning (or estimating) \citep{VV11a, OSW16, WY15, JVHW15, BZLV16} and testing \citep{valiant2017automatic, batu2001testing, batu2000testing, batu2004sublinear, indyk2012approximating} various properties of distributions. For a more in-depth overview of these topics, we refer interested readers to the comprehensive survey by \citet{canonne2020survey} and the references therein.

%In addition to the above work, there also has been some interest in incorporating the domain experts advice into the causal graph discovery process and use this advice to design improved algorithms for causal discovery.
%For instance, assuming blah lets you recover the causal graph just from the obervational data

%In most of the these works, researchers have tried to understand the optimal sample complexity needed to learn/estimate or test a certain property of the distribution. Optimality sample complexities for estimating various properties has now been well understood. For instance: support~\citep{VV11a, WY15}, support coverage~\citep{OSW16,ZVVKCSLSDM16}, entropy~\citep{VV11a, WY16, JVHW15},  distance to uniformity~\citep{VV11b, JHW16}, sorted $\ell_{1}$ distance \citep{VV11b, HJW18}, Renyi entropy~\citep{JVHW15, AOST17}, KL divergence~\citep{BZLV16, HJW16} and others. Similarly on the testing front, many optimal testers for various properties of the distribution has also been obtained. For instance: uniformity testing \citep{paninski2008coincidence}, identity testing \citep{valiant2017automatic, batu2001testing}, closeness testing \citep{batu2000testing, chan2014optimal}, monotonicity testing \citep{batu2004sublinear}, $k$-histogram testing \citep{indyk2012approximating}, independence testing \citep{batu2001testing}. For further details on these topics, we ask the reviewers to .

%~\citep{VV11a, OSW16, WY15, JVHW15, JVHW15, BZLV16}
\section{PRELIMINARIES}\label{sec:prelim}

\subsection{Graph Definitions}

Let $\cG = ([n],E)$ be a simple graph with nodes $[n]=\{1,\dots,n\}$ and edges $E$. 
% When the referred graph is unclear from the context, we use $E(\cG)$ explicitly.
A \emph{clique} is a graph where each pair of nodes are adjacent.
The \emph{degree} of a node in the graph refers to the number of adjacent nodes in the graph.
For any two nodes $i,j \in [n]$, we write $i \sim j$ if they are adjacent and $i \not\sim j$ otherwise.
The set $E$ may contain both directed and undirected edges.
To specify directed and undirected edges, we use $i \to j$ (or $j\from i$) and $i-j$\footnote{In the context of denoting paths, we sometimes use $i-j$ to represent ambiguous directions as well.} respectively.
Consider a node $i \in [n]$ in a fully directed graph $\cG$, let $\pa_\cG(i), \ch_\cG(i)$ and $\de_\cG(i)$ denote the parents, children and descendants of $i$ respectively.
Let $\cde_\cG(i) = \de_\cG(i) \cup \{i\}$.
The \emph{maximum in-degree} $d$ of $\cG$ is the size of the largest $\pa_\cG(i)$.
The \emph{skeleton} $\skel(\cG)$ of a graph $\cG$ refers to the graph where all edges are made undirected.
A \emph{v-structure} refers to three distinct nodes $i,j,k$ such that $i \to k \gets j$ and $i \not\sim j$.

A \emph{cycle} consists of $l \geq 3$ nodes where $i_1 \sim i_2 \sim\dots i_l \sim i_1$. It is directed if at least one of the edges is directed and all directed edges are in the same direction along the cycle.
A partially directed graph is a \emph{chain graph} if it has no directed cycle.
In the undirected graph obtained by removing all directed edges from a chain graph $\cG$, each connected component is called a \emph{chain component} which is a subgraph of $\cG$.
% We use $CC(\cG)$ to denote the set of all such chain components, where each chain component is a subgraph of $\cG$ and together the nodes of all chain components form the nodes of $\cG$.
For a partially directed graph, an \emph{acyclic completion} refers to an assignment of directions to undirected edges such that the resulting fully directed graph has no directed cycles.

\subsection{D-Separation and Conditional Independence}

Directed acyclic graphs (DAGs) are fully directed chain graphs that are commonly used in causality, where nodes represent random variables \citep{pearl2009causal}. Formally, consider a \emph{structural causal model} corresponding to a DAG $\cH=([n], E)$ and a set of random variables $X=\{X_1,\dots,X_n\}$ whose joint distribution $\bbP$ factorizes according to $\cH$, i.e., $\bbP(X)=\prod_{i\in[n]}\bbP(X_i\mid X_{\pa_\cH(i)})$ \citep{lauritzen1996graphical}. 
%Suppose $\cG$ is \emph{unknown}, but we can sample from the \emph{observational} distribution $\bbP(X)$. 

The factorization entails a set of conditional independencies (CIs) of the observational distribution. These entailed CI relations are fully described by a graphical criterion, known as d-separation \citep{geiger1990logic}. For disjoint sets $A,B,C\subset [n]$, sets $A$ and $B$ are \emph{d-separated} by $C$ in $\cH$ if and only if any path connecting $A$ and $B$ are inactive given $C$. A path is \emph{inactive} given $C$ when it has a non-\emph{collider}\footnote{Node $d$ is a collider on a path \emph{iff} $\cdot \to d\from\cdot$ on the path.} $c\in C$ or a collider $d$ with $\cde_\cH(d)\cap C=\varnothing$; otherwise the path is \emph{active} given $C$. We denote $A\ci B\mid_\cH C$ if $C$ d-separates $A,B$ in $\cH$ and $A\nci B\mid_\cH C$ otherwise. Fig.~\ref{fig:1} illustrates these concepts. 
\begin{figure}[h]
% \vspace{.3in}
\centering
\includegraphics[width=0.42\textwidth]{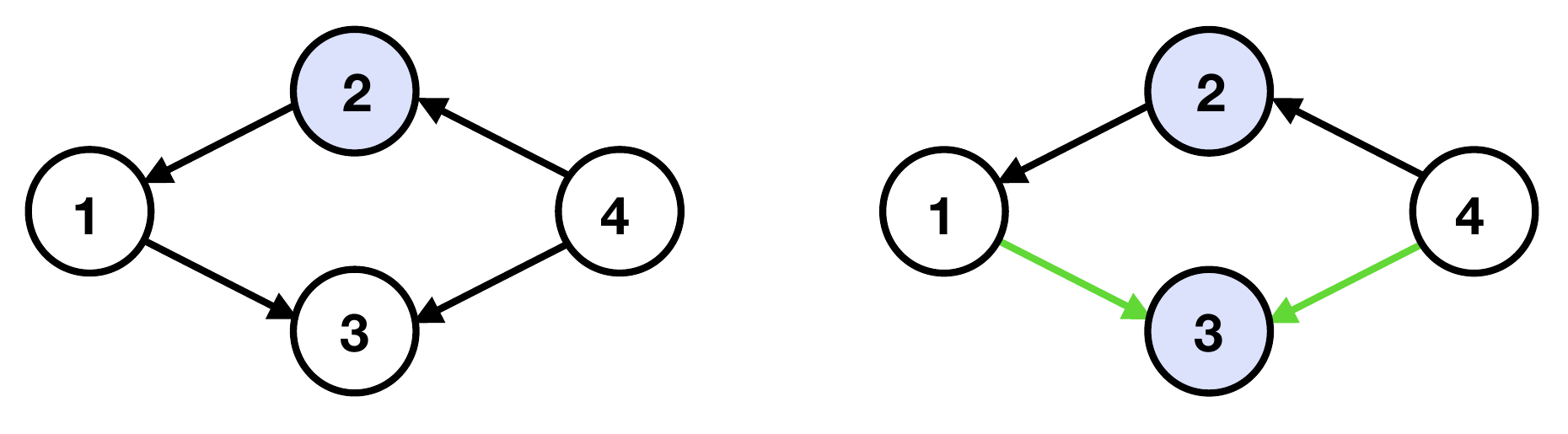}
% \vspace{.3in}
\caption{\textbf{(Left).} $\{1\}$ and $\{4\}$ are \emph{d-separated} by $\{2\}$, as both paths are \emph{inactive} given $\{2\}$. \textbf{(Right).} $\{1\}$ and $\{4\}$ are \emph{not} d-separated by $\{2,3\}$, as the path $1\to 3\from 4$ is \emph{active} given $\{2,3\}$ by the \emph{collider} $3$.}\label{fig:1}
\vspace{-.1in}
\end{figure}

Random variables $X_{A}, X_{B}$ are conditionally independent given $X_C$ if $A\ci B\mid_\cH C$ \citep{dawid1979conditional}. We write $I_\cH(A,B\mid C)=1$ if $X_{A}, X_{B}$ are conditionally independent given $X_C$ and $I_\cH(A,B\mid C)=0$ otherwise.
Under the so-called faithfulness assumption, the reverse is also true, i.e., all CI relations of $\bbP$ are implied by d-separation in $\cH$. We thus have
\[
A\ci B\mid_\cH  C \Longleftrightarrow I_\cH(A,B\mid C)=1.
\]
If any set among $A,B,C$ only contains one node, e.g., $A=\{a\}$, we write $a\ci B\mid_\cH C$ and $I_\cH(a,B\mid C)$ for simplicity.

For two DAGs $\cH$ and $\cG$, if all d-separations in $\cG$ are in $\cH$, i.e., $A\ci B\mid_\cG C\Rightarrow A\ci B\mid_\cH C$, then $\cG$ is called an \emph{independence map} of $\cH$ and write $\cH\leq\cG$. It is \emph{minimal} if removing any edge from $\cG$ breaks this relation.

\paragraph{Independence Query Oracles} In our work, we assume throughout that the causal DAG $\cH$ is \emph{unknown}. But we assume faithfulness and access to enough observational samples to determine if $X_{A}, X_{B}$ are conditionally independent given $X_C$, i.e., $I_\cH(A,B\mid C)$, when queried. We call this the \emph{independence query oracle}. 

By the above discussion, we know that the value of $I_\cH(A,B\mid C)$ equivalently implies properties of the DAG $\cH$, i.e., whether $A\ci B\mid_\cH C$. Therefore in the following, we also use $I_\cH(A,B\mid C)$ to denote that $A$ and $B$ are d-separated by $C$ in $\cH$ (with a slight misuse of notation).

% \jj{@Kiran, let me know if you think it is more suitable to put the above paragraph into the next section.}
% \jj{define d-separation, active/inactive paths, make sure to use active/inactive instead of open/closed, general setminus $A\setminus B$ does not have to be $B\subseteq A$}

\subsection{Markov Equivalence Classes}\label{sec:mec}

With observational data and no additional parametric assumptions, the DAG $\cH$ is generally only identifiable up to its Markov equivalence class (MEC) \citep{verma1990equivalence}. Two DAGs $\cH, \cG$ are in the same MEC if any positive distribution that factorizes according to $\cH$ also factorizes according to $\cG$. For any DAG $\cG$, we denote its MEC by $[\cG]$. It is known that $\cH\in[\cG]$ if and only if $\cH,\cG$ share the same skeleton and v-structures \citep{andersson1997characterization}. 
% A \emph{moral} DAG is a DAG without v-structures.
The \emph{essential graph} $\cE(\cG)$ is a partially directed graph that fully characterizes $[\cG]$, where an edge $i \to j$ is directed if $i \to j$ in \emph{every} DAG in $[\cG]$, and an edge $u \sim v$ is undirected if there exists two DAGs $\cG_1, \cG_2 \in [\cG]$ such that $i \to j$ in $\cG_1$ and $i \from j$ in $\cG_2$. We define $\pa_{[\cG]}(i)$ and $\ch_{[\cG]}(i)$ as directed parents and children of $i$ in $\cE(\cG)$, respectively. We denote $\adj_{[\cG]}(i)=\{j: j-i\in \cE(\cG)\}$ as the remaining nodes with undirected edges to $i$ in $\cE(\cG)$.

As illustrated in Fig.~\ref{fig:2}, our results are in terms of graph parameters of an MEC $[\cG]$, which are defined based on its essential graph $\cE(\cG)$. An undirected clique is a clique in $\cE(\cG)$ after removing all its directed edges, where $s$ is the size of the \emph{maximum undirected clique}.
\begin{figure}[h]
% \vspace{.3in}
\centering
\includegraphics[width=0.36\textwidth]{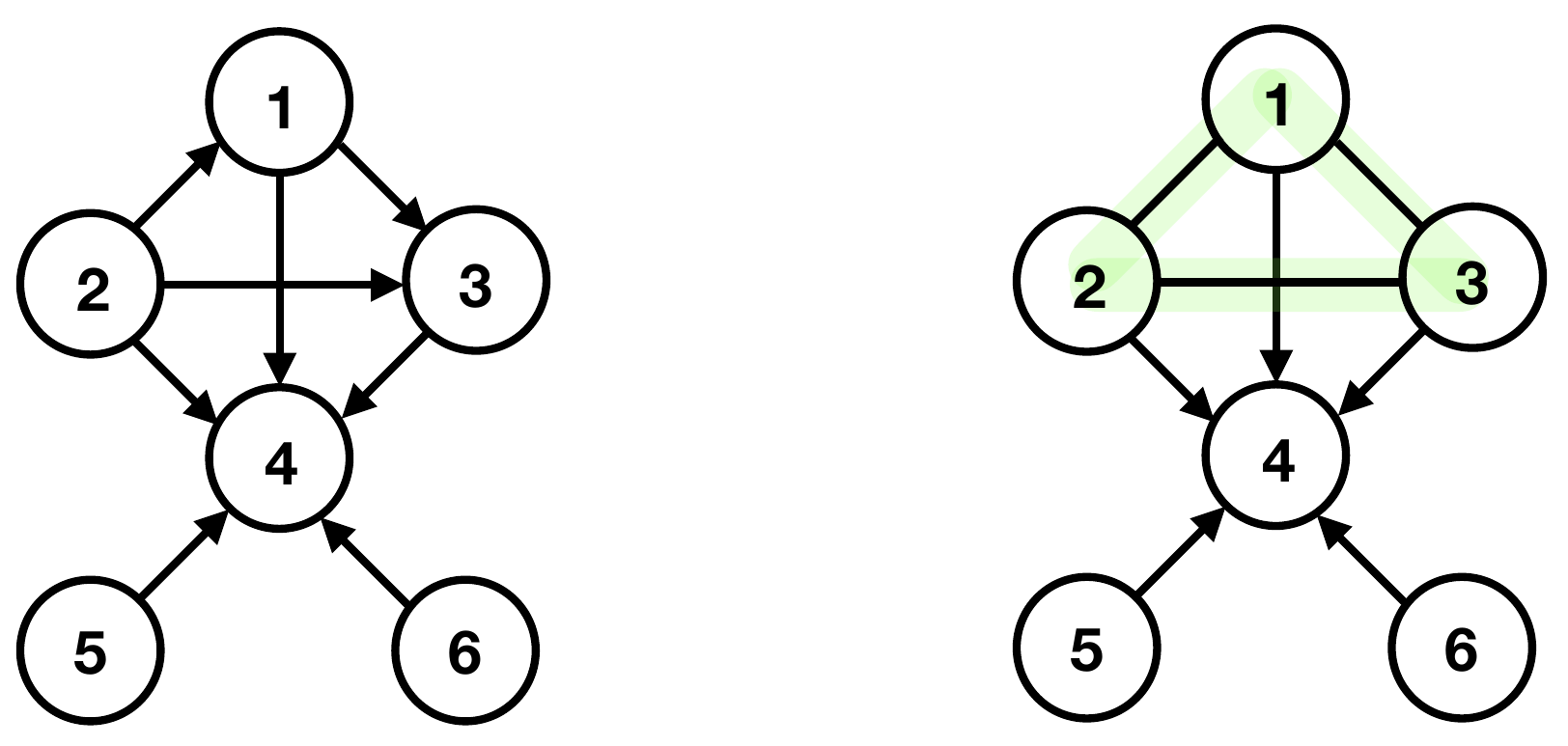}
% \vspace{.3in}
\caption{\textbf{(Left).} DAG $\cG$. \textbf{(Right).} Essential graph $\cE(\cG)$ representing $[\cG]$. In $\cE(\cG)$, the maximum undirected clique has size $s=3$ (highlighted in green). The maximum in-degree of $\cG$ is $d=5$ (on node $4$).}\label{fig:2}
% \vspace{-.1in}
\end{figure}

We now state some useful properties about essential graphs from \citet{andersson1997characterization} and \citet{wienobst2021polynomial}.
First, every essential graph is a chain graph with chordal
% \footnote{A graph is chordal if every cycle of length at least $4$ has a chord.} 
chain components. Therefore any undirected clique of the essential graph must belong to a unique chordal chain components.
Second, orientations in one chain component do not affect orientations in other components.
Third, any clique within a chain component can be made most upstream of this chain component (i.e., all edges in this chain component are pointing out from this clique) and the edge directions of this clique can be made arbitrary as long as there is no cycle within the clique.
These results imply, if we denote the maximum undirected clique of $\cE(\cG)$ by $\cS$, that for any acylic completion of $\cS$, there is $\cG_1\in[\cG]$ that contains it; moreover, $\cS$ is most upstream of $\cG_1$.

% \jj{state that MEC same skeleton+v structure. orientation in one connected component does not affect the other. for what i need later, i need to have all permutation (can i say this without incurring topological order definition?) valid for the undirected maximal clique}

%\newcommand{\cH}{\mathcal{H}}
\section{MAIN RESULTS}\label{sec:result}
As highlighted in the introduction, our work formally explores the testing aspects of causal discovery. We now provide a precise definition of the testing problem.

\paragraph{Testing Problem}
Given a specific MEC $[\cG]$ as well as independence-query-oracle access to an observational distribution $\bbP$ respecting a hidden causal DAG $\cH$, design an algorithm to test if $\cH \in [\cG]$ while minimizing the number of CI tests queried.

Throughout our work, we focus on the worst-case query complexity for causal DAGs. Our query complexity bounds are articulated in forms of the parameters of the essential graph $\cE(\cG)$, which succinctly characterizes the specified MEC $[\cG]$. In the following, we present our two key findings, which establish matching lower and upper bounds on the query complexity of the testing problem. We start with our lower bound result.

\begin{theorem}\label{thm:lowerbound}
Given a specific MEC $[\cG]$, there exists a hidden DAG $\cH$ such that any algorithm requires at least $\exp(\Omega(s))$ CI tests to test if $\cH\in[\cG]$. Here, $s$ is the size of the maximum undirected clique in $\cE(\cG)$.
\end{theorem}
% \kk{change this when you have an final bound}

It is worth noting that for the learning task, which entails the recovery of $[\cH]$, algorithms typically exhibit a query complexity that is at least an exponential function of the maximum in-degree of the essential graph. Since the size of the maximum undirected clique is always upper bounded by the maximum degree, our lower bound result suggests that the testing problem might be easier compared to learning. We confirm this by the following result, which provides an algorithm that resolves the testing problem with a query count matching the lower bound.
% of, at most, an exponential function of the size of the maximum clique.

\begin{theorem}\label{thm:upperbound}
Given a specific MEC $[\cG]$ and any hidden DAG $\cH$, there exists an algorithm that runs in polynomial \footnote{Here the run time is for choosing the next CI test to perform and is polynomial in the number of nodes.} time and performs at most $\exp(O(s)+O(\log n))$ number of CI tests to test if $\cH\in[\cG]$. 
% Here $s'$ is the size of the maximum clique in $\cE(\cG)$.
\end{theorem}
% \kk{change this when you have an final bound}

% Our upper and lower bounds rely on the sizes of the maximum clique and the maximum undirected clique, respectively. As an undirected clique is always a clique in the essential graph, the size of the maximum clique is bounded below by the size of the maximum undirected clique. These values align when the specified MEC is fully undirected, i.e., $\cG$ is moral. 
Our bounds exhibit instance-dependent \emph{tightness} for any given MEC.
% However, to achieve instance-dependent optimal bounds in general settings, we believe involves improving the upper bound to depend on the size of the maximum undirected clique. We view this as an important future research and will delve further into this in the discussion section.
% Moreover, achieving instance-dependent optimal bounds, which we believe involves enhancing the upper bound to depend on the size of the largest undirected clique, represents an important avenue for future research. We will delve further into this topic, along with discussing other future directions, in the discussion section.
In the remaining part of this section, we will provide a concise overview of the techniques used to establish our results.

\subsection{Overview of Techniques}\label{sec:overview_techniques}

% Here we will present a summary of the techniques used to derive our lower and upper bound results. We will begin by discussing the lower bound result. 
% \kk{will try to help a little with the overview, please fill in and polish it}

\paragraph{Lower Bound}
We first discuss our lower bound result for the simple case where $\cE(\cG)$ is an undirected clique. Let the hidden causal graph be obtained by removing an edge $i \to j$ from a DAG $\cG$ that belongs to the MEC, where $\pa_\cG(i)=\varnothing$. Given such $\cH$, we can show that the only set of independence test queries that differentiate $[\cH]$ from $[\cG]$ are of the form: \[i \ci j\mid \ch_{\cG}(i)\cap\pa_\cG(j).\]
As any subset of nodes in the clique could lie between $i$ and $j$ for some $\cG$ in the MEC, we immediately get a worst-case lower bound of $\binom{s}{\lceil s/2\rceil-1}=\exp(\Omega(s))$.

The above result naturally extends to MECs with general $\cE(\cG)$ by utilizing the properties that we discussed in Section~\ref{sec:mec}. Namely, any clique in a connected component of the specified MEC could be made most upstream, and therefore we could ignore all the remaining nodes in that component. A formal proof of this result is provided in Section~\ref{sec:lowerbound}.

We remark here that our result is an instance-dependent bound with respect to (any) $\cG$, which is more general than considering only fully connected $\cG$'s.

\paragraph{Upper Bound} We first show that if the specified MEC contains additional independence relations that are not in the hidden DAG, then this can be detected with $O(n^2)$ queries. 
This holds because, for each $i\not\sim j$ in the essential graph $\cE(\cG)$, one can easily find a set $C$ such that $i \ci j\mid_\cG C$. However, if $I_\cH(i,j\mid C)=0$, then this query quickly reveals $\cH\notin [\cG]$. On the other hand, if $I_\cH(i,j\mid C)=1$, then $i\ci j\mid_\cH C$. Since d-separation relations exhibit axiomatic properties, these can be used to show $\cH\leq\cG$. We call such tests canonical CI tests, which we formally define in Section~\ref{sec:canonical-ci}.

Along with these CI tests, we define another type of canonical CI tests utilitizing the undirected cliques in $\cE(\cG)$. These two types of tests together resolve the case where the hidden graph contains a missing edge in the specified MEC. In particular, if $\cH$ is missing an edge and $\cH\leq\cG$, then a result by \citet{chickering2002optimal} proving Meek's conjecture \citep{meek1995} implies that there is a DAG $\cH'$ such that $\cH\leq \cH'\leq\cG$ and $\cH'$ is one-edge away from some $\cG'\in [\cG]$. Using the \emph{local Markov property} \citep{lauritzen1996graphical}, this missing edge $i\sim j$ can be detected by $i\ci j\mid_\cH C$ for some set $C$ that contains parents of one of the nodes $i,j$. We then show that this set $C$ can be obtained via an undirected clique within $\cE(\cG)$, if the hidden graph passes the canonical CI tests defined next. Therefore we can detect this case with $\exp(O(s))$ number of CI queries. 

These results imply that we only need to consider the remaining case where $\skel(\cH)=\skel(\cG)$. When the skeleta align, we show that it is easy to test whether the v-structures align. Detailed algorithms and proofs are provided in Section~\ref{sec:upperbound}. Additionally, we provide a geometric view of these results in Section~\ref{sec:polytope}.

\subsection{Canonical CI Test Oracles}\label{sec:canonical-ci}

We now define two types of canonical CI test oracles (illustrated in Fig.~\ref{fig:3}) that will be used in our algorithms for the testing problem of whether $\cH\in[\cG]$.

\begin{definition}[Class-I CI Test]\label{def:class-i} 
For $i\not\sim j$ in $[\cG]$, 
there is $i\ci j\mid_\cG \pa_\cG(i)\setminus\{j\}$, assuming $j\notin\de_\cG(i)$ without loss of generality. Test if $I_\cH(i,j\mid \pa_\cG(i)\setminus\{j\})=1$.
\end{definition}

% We make a few remarks on the set $C(i,j)$. First, it can be computed in polynomial time. The easiest one to take is $\pa_\cG(i)$, assuming $j\notin\de_\cG(i)$ without loss of generality. Second, $C(i,j)$ must not contain any node in $\ch_\cG(i)\cap\ch_\cG(j)$. Otherwise this node will create a v-structure which constitute to an active path.
We make a few remarks on these tests. First, the d-separation claim about $\cG$ comes from the local Markov property \citep{lauritzen1996graphical}. Second, these tests are equivalent to testing if the underlying joint distribution $\bbP$ factorizes with respect to $\cG$, which is a necessary condition for $\cH\in[\cG]$.

\begin{definition}[Class-II CI Test]\label{def:class-ii}  
For $i\sim j$ in $[\cG]$, there is $i\nci j\mid_\cG \big(\pa_{[\cG]}(i)\cup C\big)\setminus\{j\}$ for all undirected cliques $C\subseteq\adj_{[\cG]}(i)$. Test if $I_\cH\big(i,j\mid\big(\pa_{[\cG]}(i)\cup C\big)\setminus\{j\}\big)=0$.
\end{definition}

\begin{figure}[t]
    \centering
     \begin{subfigure}{0.43\textwidth}
         \centering
         \includegraphics[width=\textwidth]{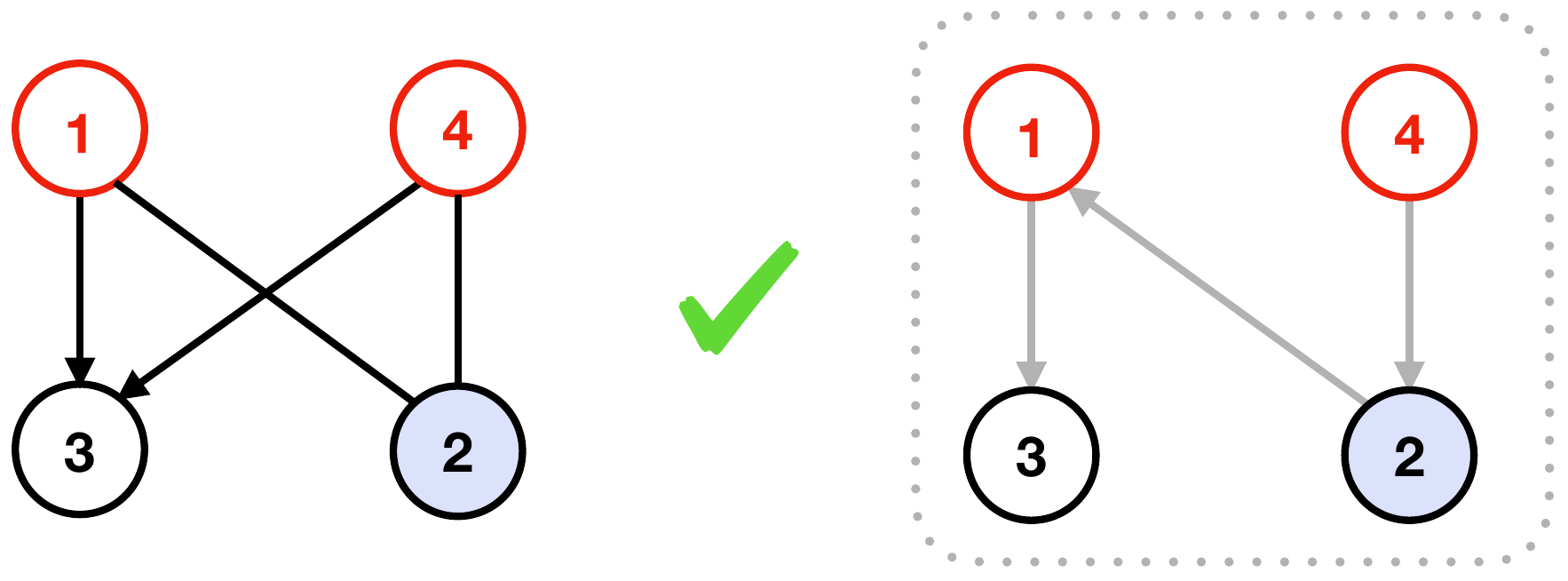}
         \caption{Class-I CI Test}
     \end{subfigure}
     \hfill
     \begin{subfigure}{0.43\textwidth}
         \centering
         \includegraphics[width=\textwidth]{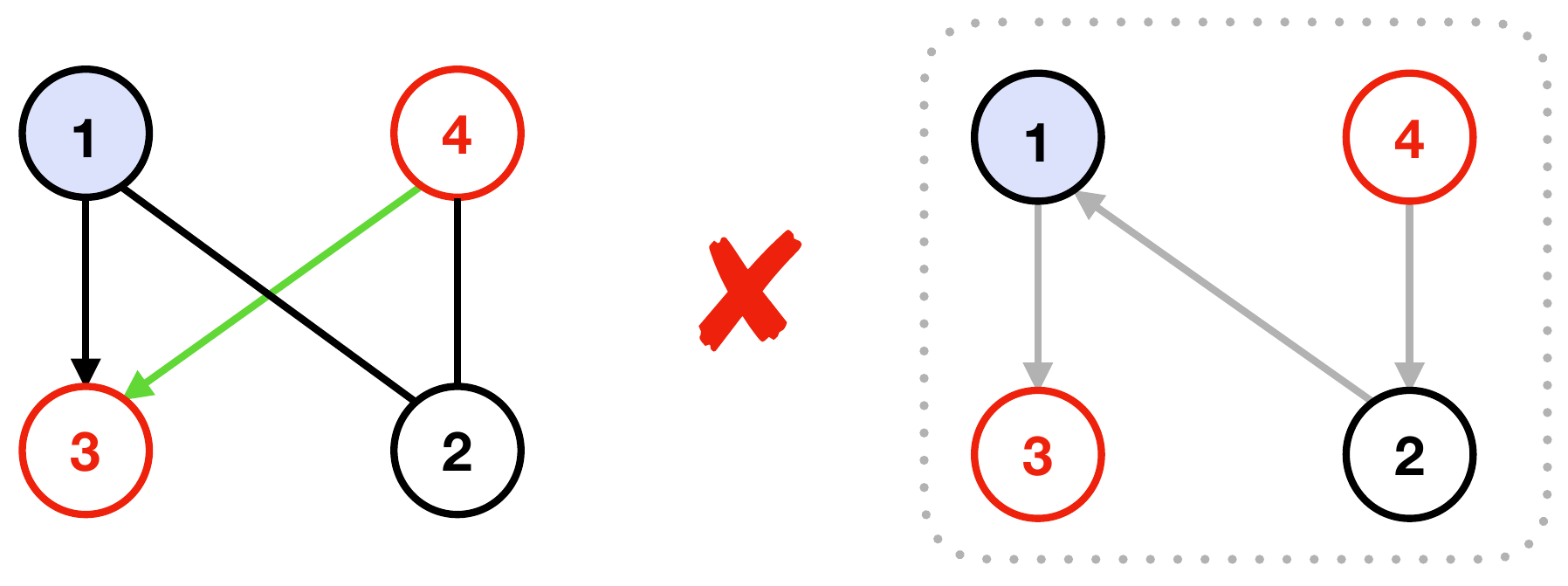}
         \caption{Class-II CI Test}
     \end{subfigure}     
\caption{Examples of canonical CI tests. The given MEC $\cG$ is on the left, and the hidden $\cH$ is on the right. \textbf{(a)}~Class-I CI test $1\ci 4\mid 2$ agrees between $\cH,\cG$. \textbf{(b)}~Class-II CI test $3\ci 4\mid 1$ disagrees between $\cH,\cG$.}\label{fig:3}
\end{figure}

% \textcolor{red}{This paragraph is too unspecific} 
We will show in Section~\ref{sec:upperbound} that if all the class-I and class-II canonical independence queries are satisfied in the hidden graph $\cH$, then $\cH$ has to be in $[\cG]$.

\section{LOWER BOUND}\label{sec:lowerbound}

In Section~\ref{sec:overview_techniques}, we explained how the lower bound example is constructed; namely by considering a DAG $\cH$ that is missing one undirected edge in $[\cG]$. We now provide steps towards Theorem~\ref{thm:lowerbound}. All omitted proofs can be found in Appendix~\textcolor{red}{2}.

A key lemma is to show that when two DAGs are very similar, namely, one is missing only an edge that is undirected in the MEC of the other, they share certain active/inactive paths.

\begin{lemma}\label{lm:2dag1var}
    Let $\cG_1$ and $\cG_2$ be two DAGs such that $\cG_1$ differs from $\cG_2$ \textup{only} by missing one edge $i\to j$, which is undirected in $\cE(\cG_2)$. Let $\cP:a-\dots-b$ be a common path shared by $\cG_1$ and $\cG_2$. For any arbitrary set $C\subseteq [n]\setminus\{a,b\}$, 
    \begin{enumerate}
        \vspace*{-0.12in}
        \item[(1)] if $\cP$ is \textup{inactive} given $C$ in $\cG_2$, then it must be \textup{inactive} given $C$ in $\cG_1$; 
        \vspace*{-0.08in}
        \item[(2)]  if $\cP$ is \textup{active} given $C$ in $\cG_2$, then there is a path connecting $a,b$ that is \textup{active} given $C$ in $\cG_1$.
        \vspace*{-0.08in}
    \end{enumerate}
\end{lemma}

 An immediate corollary of this lemma is the following statement about d-separations in such DAGs.

\begin{corollary}\label{cor:2dag1var-c}
     Let $\cG_1,\cG_2$ be two DAGs as in Lemma~\ref{lm:2dag1var}. For any nodes $a\neq b\in [n]$ and set $C\subset [n]\setminus\{a,b\}$, 
    \begin{enumerate}
        \vspace*{-0.12in}
        \item[(1)] if $a\ci b\mid_{\cG_2} C$, then $a\ci b\mid_{\cG_1} C$;
        \vspace*{-0.08in}
        \item[(2)]  if $a\nci b\mid_{\cG_2} C$ and there is path from $a$ to $b$ in $\cG_1$ that is active given $C$ in $\cG_2$, then $a\nci b\mid_{\cG_1} C$. 
        \vspace*{-0.08in}
    \end{enumerate}  
\end{corollary}

\begin{proof} 
If $a\ci b\mid_{\cG_2} C$, then all paths from $a$ to $b$ are inactive given $C$ in $\cG_2$. Since every path in $\cG_1$ is a path in $\cG_2$, all paths from $a$ to $b$ are inactive given $C$ in $\cG_1$ as well by Lemma \ref{lm:2dag1var} (1). Thus $a\ci b\mid_{\cG_1} C$.

If $a\nci b\mid_{\cG_2} C$ and there is path from $a$ to $b$ in $\cG_1$ that is active given $C$ in $\cG_2$, this would be a common path shared by $\cG_1,\cG_2$. By Lemma \ref{lm:2dag1var} (2), there must be a path connecting $a,b$ that is active given $C$ in $\cG_1$. Thus $a\nci b\mid_{\cG_1} C$.
\end{proof}

Using Corollary~\ref{cor:2dag1var-c}, we can now show that if a CI test disagrees between such two similar DAGs, then the conditioned set must intersect a particular neighborhood of the missing edge.

\begin{lemma}\label{lm:2dag1ineq}
Let $\cG_1,\cG_2$ be two DAGs that differ by a missing edge $i\to j$ as in Lemma~\ref{lm:2dag1var}. Denote the maximal undirected clique in $\cG_2$ containing $i,j$ by $\cS$. Then $I_{\cG_1}(A,B\mid C)\neq I_{\cG_2}(A,B\mid C)$ \textup{only} if $A\ci B\mid_{\cG_1} C$, $A\nci B\mid_{\cG_2} C$, and 
\begin{align*}
\left(\pa_{\cG_2}(j)\cap\ch_{\cG_2}(i)\right) \cap \cS &\subseteq\\ C\cap \cS &\subseteq \left(\pa_{\cG_2}(j)\setminus\{i\}\right) \cap \cS.
\end{align*}
\end{lemma}

This result lays the crucial step for Theorem~\ref{thm:lowerbound}, since the particular neighborbood of the missing edge can require $\exp(\Omega(s))$ tests to be identified in the worst case.

\begin{proof}[Proof of Theorem~\ref{thm:lowerbound}.]
Let $\cS$ be the maximum undirected clique in $\cE(\cG)$ with size $s$. Denote $i,j$ as the $1$-st and $\lfloor s/2\rfloor$-th nodes in the topological order\footnote{The \emph{topological order} $\pi : [n] \to [n]$ associated to a DAG $\cG$ is such that any $i\to j$ in $\cG$ satisfies $\pi(i) < \pi(j)$.} of nodes of $\cS$ in the DAG $\cG$. Let $K$ be the set of all nodes that lie in between $i$ and $j$ in the topological order.

Let $\cH$ be the DAG obtained by removing $i\to j$ from $\cG$. We will show that any algorithm requires at least $\binom{s}{\lceil s/2\rceil-1}$ CI tests to verify $\cH\notin[\cG]$ in the worst case.

Since $\cS$ is undirected, $i\to j$ must be undirected in $\cE(\cG)$. 
Using Lemma~\ref{lm:2dag1ineq} for $\cG$ and $\cH$, any disjoint sets $A,B,C$ satisfy that $I_{\cH}(A,B\mid C)\neq I_{\cG}(A,B\mid C)$ only if $C\cap \cS=K$ (i.e., the nodes of $\cS$ that are in $C$ are exactly $K$), since
\begin{align*}
K = \left(\pa_{\cG}(j)\cap\ch_{\cG}(i)\right) \cap \cS \subseteq C\cap \cS &\\ \subseteq \left(\pa_{\cG}(j)\setminus\{i\}\right) \cap \cS & = K.
\end{align*}
Therefore any CI test of $A,B$ given $C$ would agree between $\cH$ and $\cG$ if $C\cap \cS\neq K$. However, since $\cS$ is an undirected clique in $\cE(\cG)$, its nodes can be ordered arbitrarily to obtain a valid DAG in $[\cG]$. Therefore, no additional information can be learned about $i,j$ by performing CI tests until $C\cap \cS= K$. Since all topological orders can be valid, it can take $\binom{s}{|K|}=\binom{s}{\lceil s/2\rceil-1}$ CI tests until $C\cap \cS= K$ in the worst case.
\end{proof}

We make some final remarks about Theorem~\ref{thm:lowerbound}. First, it is a worst-case result over all possible hidden graphs $\cH$. Second, since $s$ depends on $[\cG]$, this lower bound is instance-wise with respect to the MEC of interest.

\section{UPPER BOUND}\label{sec:upperbound}

We now present our upper bound results. All omitted proofs can be found in Appendix~\textcolor{red}{3}.

We begin by presenting our algorithm for testing. Our algorithm is built upon the canonical CI tests introduced in Section~\ref{sec:canonical-ci} (Definitions~\ref{def:class-i},\ref{def:class-ii}).

\begin{algorithm}[tbh]
\begin{algorithmic}[1]
\caption{Membership Testing in MEC.}
\label{alg:1}
    \State \textbf{Input}: MEC $[\cG]$ and independence-query oracle access to hidden DAG $\cH$.
    \State \textbf{Output}: whether hidden $\cH$ belongs to $[\cG]$.
    \State Perform all class-I CI tests with respect to $\cG$ and $\cH$ sequentially; \textbf{return} False once $\cH$ fails.
    \State \textbf{for} $i\sim j$ in $[\cG]$ \textbf{do}
        \State\hspace{10pt}\textbf{for} undirected clique $C$ in $\adj_{[\cG]}(i)$ \textbf{do}
        \State\hspace{20pt} Test $I_\cH\big(i,j\mid\big(\pa_{[\cG]}(i)\cup C\big)\setminus\{j\}\big)$. \State\hspace{20pt} \textbf{return} False if independence.
    \State\hspace{10pt} Perform line 4-6 for $j$.   
    \State \textbf{return} True.
\end{algorithmic}
\end{algorithm}

Note that the total number of maximal undirected cliques in $[\cG]$ cannot exceed $n$ (see Appendix~\textcolor{red}{3}). Since each undirected clique must belong to some maximal undirected clique whose size $\leq s$, the total number of undirected cliques cannot exceed $n\cdot 2^{s}$. Thus, the total number of class-II CI tests performed is bounded by $n^3\cdot 2^{s}=\exp(O(s)+O(\log n))$.
% As sets $\pa_{[\cG]}(i),\ch_{[\cG]}(i),\adj_{[\cG]}(i)$ are disjoint, the total number of class-III CI tests can not exceed $n^5\cdot 2^{s'}=\exp(O(s')+O(\log n))$, where $s'$ is the size of maximum clique in $[\cG]$.
% As the size of undirected clique in $[\cG]$ does not exceed $s$, the total number of undirected cliques in $[\cG]$ can not exceed $n+\dots+n^s=\exp()$. Therefore the total number of class-III CI tests can not exceed $\exp(O(s)+O(\log n))$.

We prove the correctness of Algorithm~\ref{alg:1} by showing that if $\cH\notin[\cG]$, then $\cH$ fails at least one of the class-I or class-II CI tests. It is clear from the definition that $\cH$ will pass all these CI tests when $\cH\in[\cG]$.

\subsection{Class-I CI Tests Imply $\cH\leq \cG$}\label{sec:class-i-indp-map}

We first show that passing class-I CI tests implies $\cH\leq \cG$, i.e., all CI relations in $\cG$ must hold in $\cH$.

\begin{lemma}\label{lm:class-I-imply-ind-map}
    If $\cH$ passes all class-I CI tests, then $\cH\leq\cG$. In particular, this implies $\skel(\cH)\subseteq \skel(\cG)$.
\end{lemma}

\begin{proof}
Passing all class-I CI tests with respect to $\cG$ and $\cH$ implies that the joint distribution $\bbP$ factorizes according to $\cG$, which in turn implies that $\bbP$ satisfies all CI relations given by d-separation in $\cG$. Since $\bbP$ is faithful to $\cH$, it follows that $\cH\leq \cG$. Additionally, if there is an edge $i\sim j$ in $\cH$ but not in $\cG$, assuming $j\not\in\de_\cG(i)$, we obtain $i\ci j\mid_\cG \pa_\cG(i)\setminus\{j\}$ but $i\nci j\mid_\cH \pa_\cG(i)\setminus\{j\}$; a contradiction to $\cH\leq\cG$.
\end{proof}

Note that the consequence of class-I CI tests in Lemma~\ref{lm:class-I-imply-ind-map} essentially follows from the equivalence between the local and global Markov properties~\citep{pearl1988probabilistic}, which establishes that all d-separation statements can be deduced by only a few (local) statements.

% In general, with arbitrary class-I CI tests,  passing them has the following implied property: for any triplets $i\sim j\sim k$ such that $i\not\sim j$ in $\cH$, there is (1) if $k$ is a collider on $i\sim k\sim j$ in $\cH$, then either it is also a collider or $i\sim j$ in $\cG$; (2) if $k$ is not a collider on $i\sim k\sim j$ in $\cH$, then either it is not a collider or $i\sim j$ in $\cG$. Then the proof of Lemma~\ref{lm:class-I-imply-ind-map} proceeds by considering the minimal Bayes-ball path \citep{shachter2013bayes} between any two nodes that are not adjacent in $\cG$. 
% We defer the formal proof to Appendix~\textcolor{red}{3}.

\subsection{Class-II CI Tests Imply $\skel(\cH)=\skel(\cG)$}

Suppose that $\cH$ passes all class-I CI tests; we now show that passing all class-II CI tests implies $\skel(\cH)=\skel(\cG)$.

Using Lemma~\ref{lm:class-I-imply-ind-map}, we obtain $\cH\leq\cG$ and $\skel(\cH)\subseteq\skel(\cG)$. If $\skel(\cH)\neq\skel(\cG)$, then there exists a CI relation that holds in $\cH$ but not in $\cG$. In this case, we can construct a ``middle'' DAG that lies between $\cH,\cG$ whose skeleton differs from $\skel(\cG)$ by only one edge. This construction is a direct consequence based on the proof by \citet{chickering2002optimal} of Meek's conjecture \citep{meek1995}. 
% See Appendix~\textcolor{red}{3} for details.

\begin{proposition}\label{prop:chickering}
    If $\cH\leq\cG$ and $\skel(\cH)\subsetneq \skel(\cG)$, then there exist a DAG $\cH'$ such that $\cH\leq\cH'\leq\cG$ and $\cH'$ is missing one edge in $\cG'$ for some $\cG'\in [\cG]$.
\end{proposition}

Then it only remains to find the existence of a CI relation that holds in $\cH'$ (which holds in $\cH$ since $\cH\leq \cH'$) but not in $\cG'$ (or equivalently, $\cG$, since $\cG'\in[\cG]$). This can be detected by class-II CI tests. 
The intuition for this is provided in the following lemma.
\begin{lemma}\label{lm:class-II-intuition}
    For any $i$ in $\cG'\in[\cG]$, there is $\pa_{\cG'}(i)=\pa_{[\cG]}(i)\cup C$ for some undirected clique $C\subseteq\adj_{[\cG]}(i)$.
\end{lemma}

An immediate consequence of these two results is the following corollary.

\begin{corollary}\label{cor:class-II-adj}
    If $\cH$ passes all class-I and class-II CI tests, then $\skel(\cH)=\skel(\cG)$.
\end{corollary}

\begin{proof}
    Suppose $\skel(\cH)\neq\skel(\cG)$; then by Lemma~\ref{lm:class-I-imply-ind-map} and Proposition~\ref{prop:chickering}, we know there exists $\cH'$ and $\cG'\in[\cG]$ such that $\cH\leq \cH'\leq \cG$ and $\cH'$ differs from $\cG'$ by one missing edge. Suppose the missing edge is $j\to i$ in $\cG'$, then $i\ci j\mid_{\cH'} \pa_{\cH'}(i)$. Note that $\pa_{\cH'}(i)=\pa_{\cG'}(i)\setminus\{j\}$. By Lemma~\ref{lm:class-II-intuition}, $\pa_{\cG'}(i)=\pa_{[\cG]}(i)\cup C$ for some undirected clique $C\subseteq\adj_{[\cG]}(i)$. Therefore $i\ci j\mid_{\cH'} \pa_{[\cG]}(i)\cup C\setminus\{j\}$. Since $\cH\leq\cH'$, we have $i\ci j\mid_{\cH} \pa_{[\cG]}(i)\cup C\setminus\{j\}$. This means that $\cH$ fails a class-II CI test; a contradiction. 
\end{proof}

We are now ready to prove our main theorem.

\begin{proof}[Proof of Theorem~\ref{thm:upperbound}]
If $\cH$ fails any of the class-I or class-II tests, then $\cH\notin[\cG]$. Thus, suppose that $\cH$ passes all the class-I and class-II tests. Then by Corollary~\ref{cor:class-II-adj}, we obtain $\skel(\cH)=\skel(\cG)$. For any triplets $i\sim k\sim j$ such that $i\not\sim j$ in $\cG$, assume that $j\not\in\de_\cG(i)$ without loss of generality. If $i\sim k\sim j$ is a v-structure in $\cG$, then $k\not\in \pa_\cG(i)$. Consequently, it is a v-structure in $\cH$, since it passes the class-I test of $i\ci j\mid_\cH \pa_\cG(i)\setminus\{j\}$ in Definition~\ref{def:class-i}. Similarly if it is not a v-structure in $\cG$, then it is not a v-structure in $\cH$. Thus $\cH,\cG$ also share the same set of v-structures. Hence it holds that $\cH\in[\cG]$. 

Since the total number of class-I and class-II tests sum up to $\exp(O(s)+O(\log n))$, we arrive at the result.
\end{proof}

\section{DAG ASSOCIAHEDRON}\label{sec:polytope}

In this section, we map our findings onto the DAG associahedron \citep{mohammadi2018generalized}, thereby providing a geometric interpretation of our results. Using edgewalks on the DAG associahedron, we establish in Section~\ref{sec:6.1} that (1) testing is stricly harder than learning, (2) how our testing algorithm can aid learning with potentially fewer CI tests.

We begin by introducing the DAG associahedron. Let $\cA_n$ denote the \emph{permutohedron} on $n$ elements, i.e., the convex polytope in $\bbR^n$ with vertices corresponding to all permutation vectors of size $n$. Two vertices on $\cA_n$ are connected by an edge if and only if their corresponding permutations differ by an adjacent transposition. The \emph{DAG associahedron} $\cA_n(\cH)$ with respect to a DAG $\cH$ is obtained by contracting edges in $\cA_n$ corresponding to d-separations in $\cH$.\footnote{Note that this does not require knowing $\cH$ but only its d-separations.} Namely, the edge between vertices $(\pi_1,\dots,\pi_i,\pi_{i+1},\dots,\pi_n)$ and $(\pi_1,\dots,\pi_{i+1},\pi_{i},\dots,\pi_n)$ is contracted when $\pi_i\ci\pi_{i+1}\mid_\cH \{\pi_1,\dots,\pi_{i-1}\}$. It was shown that (1) $\cA_n(\cH)$ remains a convex polytope (section 4 in \citet{mohammadi2018generalized}); (2) the vertices of $\cA_n(\cH)$ correspond to minimal independence maps of $\cH$ that can be obtained via any of the associated permutations before contraction (Theorem 7.1 in \citet{mohammadi2018generalized}). Fig.~\ref{fig:4} illustrates these concepts.

\begin{figure}[h]
% \vspace{.3in}
\centering
\includegraphics[width=0.44\textwidth]{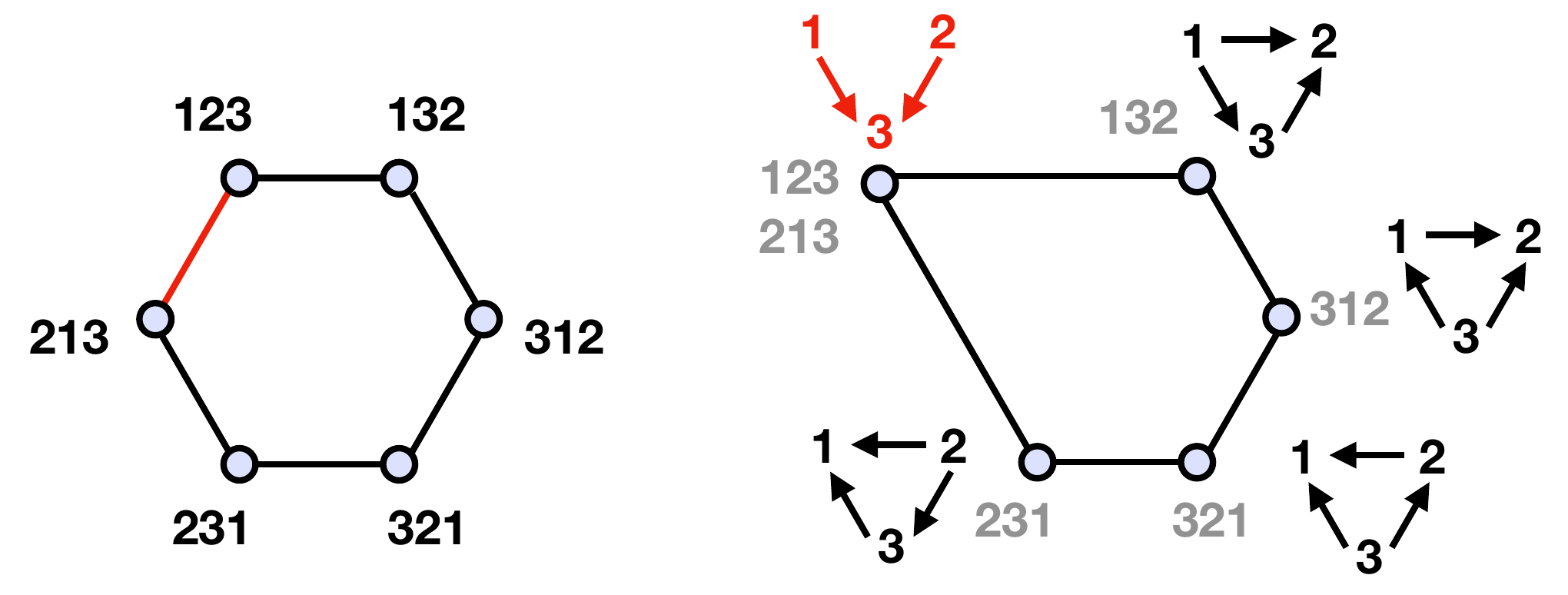}
% \vspace{.3in}
\caption{\textbf{(Left).} Permutohedron $\cA_3$. \textbf{(Right).} DAG associahedron $\cA_3(1\to 3\from 2)$. The corresponding DAG and contracted edge are in red.}\label{fig:4}
% \vspace{-.1in}
\end{figure}

We can now reinterpret our results using $\cA_n(\cH)$.

\paragraph{Testing if $\cG$ is on the polytope} Note that as $\cH$ and any DAGs in $[\cH]$ are minimal independence maps of $\cH$, property (2) in the above paragraphs indicates that $\cH\in[\cG]$ only if $\cG$ is a vertex of $\cA_n(\cH)$. Therefore the first test is to see if $\cG$ is on $\cA_n(\cH)$.
By Lemma~\ref{lm:class-I-imply-ind-map}, we can test if $\cG$ is an independence map of $\cH$ by class-I CI tests. To further test if $\cG$ is a minimal independence map of $\cH$, one only needs to perform $O(n^2)$ tests to see if any edge in $\cG$ is removable. Therefore this gives us a way to rule out the case when $\cG$ is not on $\cA_n(\cH)$.

\paragraph{Testing if $\cG$ is a sparsest DAG} Once we establish that $\cG$ is on $\cA_n(\cH)$, we can test if $\skel(\cG)=\skel(\cH)$ by testing if $\cG$ is a sparsest DAG on $\cA_n(\cH)$ (since no minimal independence map of $\cH$ can be sparser than $\skel(\cH)$). 
If $\cG$ is not sparsest, then it was shown that one can find a strictly sparser DAG $\cH_1$ such that $\cH\leq\cH_1\leq\cG$ by a specific sequence of edgewalks from $\cG$ on $\cA_n(\cH)$ (Proposition 8(b) in \citet{solus2021consistency}). On the contrary, if no such edgewalks exists, then one can conclude that $\cG$ is sparsest. Concretely, each edgewalk corresponds to flipping a covered edge in the starting DAG, obtaining a topological order of the flipped DAG, and finding the minimal independence map of this topological order via $O(n^2)$ CI tests. 

Note that this existence result does not imply any non-trivial upper bounds on the number of edgewalks required before finding a sparser DAG. One way to see this is by considering the neighborhood of a sparsest $\cH$. Since all DAGs in $[\cH]$ are on $\cA_n(\cH)$ and one can traverse from one to another via a sequence of covered edge flips \citep{chickering1995transformational}, they are all connected on $\cA_n(\cH)$ via the aforementioned edgewalks. Therefore a trivial computation following the above strategy sums up to $O(n^2)\cdot|[\cH]|$ CI tests for verifying that $\cH$ is sparsest. Here, $|[\cH]|$ is the number of DAGs in the MEC, which can be $\exp(\Omega(n))$ regardless of the maximal undirected clique size $s$ (see Appendix~\textcolor{red}{4}).\footnote{This excludes the trivial case where $s=1$.}

In comparison, our results in Proposition~\ref{prop:chickering} and Lemma~\ref{lm:class-II-intuition} establish that we can skip a lot of edgewalks by directly performing class-II CI tests, which entails no more than $\exp(O(s)+O(\log n))$ CI tests. Upon testing if $\skel(\cH)=\skel(\cG)$, it is easy to test if $\cH\in[\cG]$ following our proof of Theorem~\ref{thm:upperbound}.

\subsection{Learning vs. Testing}\label{sec:6.1}

By the discussion above, the problem of learning $[\cH]$ can be seen as identifying a sparsest vertex of the DAG associahedron $\cA_n(\cH)$, whereas the problem of testing if $[\cG]=[\cH]$ corresponds to verifying if $\cG$ is a sparsest vertex. In this regard, it is evident that testing is strictly easier than learning, unless one stumbled upon the correct $[\cH]$ at the initial round of learning.

Note that when $\cG$ is not sparsest, the existence results discussed above establish that one can walk along the edges of $\cA_n(\cH)$ to a strictly sparser DAG. Therefore one can build a greedy search algorithm over $\cA_n(\cH)$ to learn $[\cH]$. GSP \citep{solus2021consistency} does this by searching for the sparsest permutations; \citet{lam2022greedy} and \citet{andrews2023fast} showed that certain edgewalks can be skipped by considering traversals of permutations that are different from GSP. In comparison, our results indicate that instead of edgewalks on the DAG associahedron, one can directly arrive at a strictly sparser DAG via class-II CI tests.\footnote{Consider the minimal independence map obtained by a topological order of $\cH'$ in Proposition~\ref{prop:chickering}.} When the starting DAG belongs to an MEC of large size but has very small undirected cliques, these tests can be much more efficient than edgewalks. Thus in these cases, adopting our strategy may aid learning with potentially fewer CI tests.
\section{DISCUSSION}\label{sec:discuss}

In this work, we introduced the testing problem of causal discovery. We established matching lower and upper bounds on the number of conditional independence tests required to determine if a hidden causal graph, which can be queried using conditional independence tests, belongs to a specified Markov equivalence class. There are several interesting future directions stemming from this work. These include deriving bounds that generalize our results to cyclic graphs. While our work focused on testing if a hidden causal graph belongs to a specified MEC using observational data, it would also be of interest to explore extensions of testing in the presence of interventional data.

Furthermore, it would also be valuable to establish sample complexity bounds and statistical evaluations for the testing problem. As our results provide a binary answer, it will be relevant  for real-world scenarios to establish non-binary scores, e.g., measuring how well the given MEC represents the hidden DAG. Another problem which aligns naturally with traditional property testing literature is the approximate testing problem: testing if the hidden graph is in the given MEC or $\epsilon$-far-away from it. There could be many different ways for defining $\epsilon$-far-away distances (such as SHD \citep{acid2003searching,tsamardinos2006max} and SID\citep{peters2015structural}) that are of interest.

Finally, it would be interesting to explore the implications of our results for causal structure learning. This would be particularly relevant in the context of algorithms that perform greedy search either in the space of permutations (such as GSP \citep{solus2021consistency} and extensions thereof \citep{lam2022greedy,andrews2023fast}) or in the space of MECs (such as GES \citep{chickering2002optimal}).

% Learning:  GES \citep{chickering2002optimal}, an earlier score-based approach, does this by searching amongst MECs that are independence maps.

%In our work, we introduced the testing aspect of causal discovery and provided a lower and upper bound on the number of conditional independence test required to test if the underlying hidden cuasal graph belongs to a specified MEC. Our uppwer and lower bounds match only on MEC's where the undirected and clique sizes coincide. An immeidate interesting open question is to provide instance dependent matching upper and lower bounds, which we think entails improving the upper bound. Other interesting directions to our work include derviing bounds for cyclic graphs and also relaxing the a assumptions made in our work, in particular the causal sufficiency assumption. It is also interesting to derive sample complexity bounds for the testing problem making assumptions on the distribution which is generating the data. Although our work focused on testing if the hidden causal graph belongs to an MEC under obervational data, it is also interesting to extend the definition of testing in the presence of interventional data. 

%and ask the following question. Given a specified interventional MEC, test if the hidden causal graph belongs to the it from the observational and interventional data.

\subsection*{Acknowledgements}
We thank the anonymous reviewers for helpful comments. J.Z.~was partially supported by an Apple AI/ML PhD Fellowship.
K.S.~ was supported by a fellowship from the Eric and Wendy Schmidt Center at the Broad Institute. The authors were partially supported by NCCIH/NIH (1DP2AT012345), ONR (N00014-22-1-2116), the United States Department of Energy (DOE), Office of Advanced Scientific Computing Research (ASCR), via the M2dt MMICC center (DE-SC0023187), the MIT-IBM Watson AI Lab, and a Simons Investigator Award to C.U.

\bibliographystyle{apalike}
\bibliography{ref}

% \clearpage
% \tableofcontents
% \addtocontents{toc}{\protect\setcounter{tocdepth}{2}}

\newpage
\appendix
\section{PRELIMINARIES}

We remark here that we do not assume causal sufficiency, since the testing problem assumes that the joint distribution is Markov and faithful to some DAG $\cH$ (which does not necessarily imply causal sufficiency) and asks if this DAG is contained in a given MEC $[\cG]$. This is in contrast to the learning problem, where one cares about the complete causal explanation \citep{meek1997graphical} and therefore needs to assume e.g., no unobserved causal variables or cycles.

Since all the DAGs in the same Markov equivalence class share the same v-structures, we know that these v-structures are directed in $\cE(\cG)$. In addition, we can orient an additional set of edges using a set of logical relations known as Meek rules \citep{meek1995}.

\begin{proposition}[Meek Rules \citep{meek1995}]\label{prop:1}
We can infer all directed edges in $\cE(\cG)$ using the following four rules:
\begin{enumerate}
    \item If $i\rightarrow j \sim k$ and $i\not\sim k$, then $j\rightarrow k$.
    \item If $i\rightarrow j \rightarrow k$ and $i\sim k$, then $i\rightarrow k$.
    \item If $i\sim j, i\sim k, i\sim l, j\rightarrow k, l\rightarrow k$ and $j\not\sim l$, then $i\rightarrow k$.
    \item If $i\sim j, i\sim k, i\sim l, j\leftarrow k, l\rightarrow k$ and $j\not\sim l$, then $i\rightarrow j$.
\end{enumerate}
\end{proposition}

Figure~\ref{fig:a2} illustrates Meek rule 1, which we will use in our lower bound proof.
\begin{figure}[h]
% \vspace{.3in}
% \vspace{-.3in}
\centering
\includegraphics[width=0.25\textwidth]{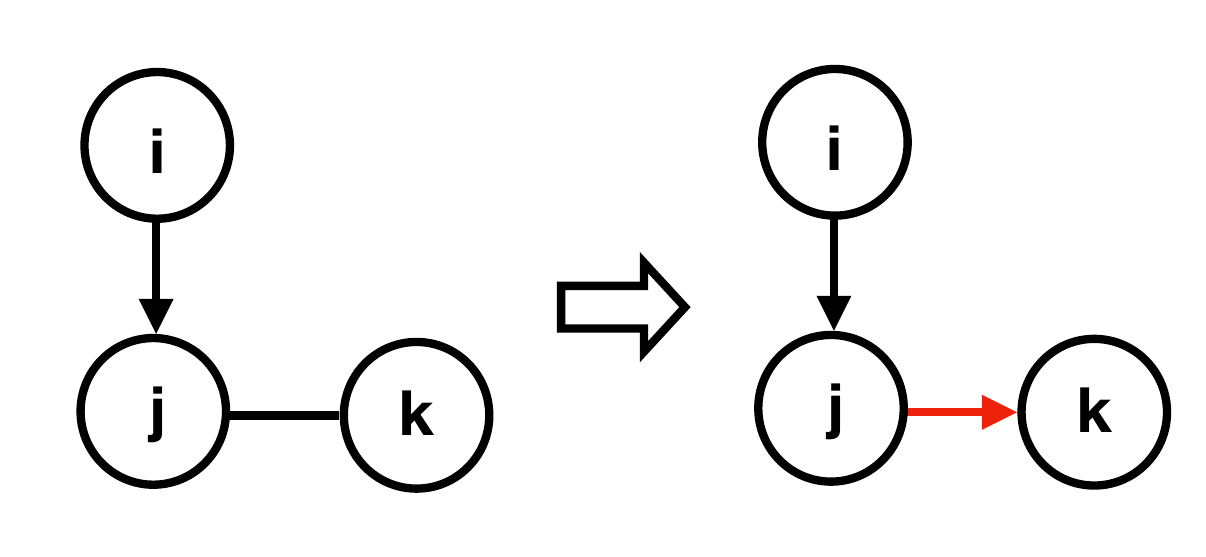}
% \vspace{.3in}
\caption{Illustration of Meek rule 1.}\label{fig:a2}
\vspace{-.1in}
\end{figure}

\section{OMITTED PROOFS OF LOWER BOUND}

% \kk{Can we solve approximate testing problem? still requires formalization}

\subsection{Proof of Lemma~\textcolor{red}{5}}
We now prove Lemma~\textcolor{red}{5}, restated below:

\begin{namedlemma}[Lemma 5]\label{lm:2dag1var}
    Let $\cG_1$ and $\cG_2$ be two DAGs such that $\cG_1$ differs from $\cG_2$ \textup{only} by a missing edge $i\to j$, which is undirected in $\cE(\cG_2)$. Let $\cP:a-\dots-b$ be a common path shared by $\cG_1$ and $\cG_2$. For any arbitrary set $C\subset [n]\setminus\{a,b\}$, 
    \begin{enumerate}
        \item[(1)] if $\cP$ is inactive given $C$ in $\cG_2$, then it must be inactive given $C$ in $\cG_1$;
        \item[(2)] if $\cP$ is active given $C$ in $\cG_2$, then there is a path connecting $a,b$ that is active given $C$ in $\cG_1$.
    \end{enumerate}
\end{namedlemma}

\begin{proof}
Since $\cG_1$ differs from $\cG_2$ by missing one edge and $\cP$ is a shared common path, we know that $\cP$ has the same set of colliders (and non-colliders) in $\cG_1$ and $\cG_2$.

We first show (1). If $\cP$ is inactive given $C$ in $\cG_2$, then either there exists a non-collider $c\in \cP\cap C$ or there exists a collider $d\in\cP$ such that $\cde_{\cG_2}(d)\cap C=\varnothing$. In the first case, $c\in\cP\cap C$ is a non-collider in $\cG_1$ as well. Thus $\cP$ is inactive given $C$ in $\cG_1$. In the second case, since $\cG_1$ and $\cG_2$ differ only by one missing edge, it holds that $\cde_{\cG_1}(d)\subseteq \cde_{\cG_2}(d)$. Therefore $d\in\cP$ is a collider in $\cG_1$ and $\cde_{\cG_1}(d)\cap C\subseteq \cde_{\cG_2}(d)\cap C=\varnothing$. Thus $\cP$ is inactive given $C$ in $\cG_1$, which proves (1).

We now show (2). Assume to the contrary that $\cP$ is active given $C$ in $\cG_2$ and that there is \emph{no} path connecting $a,b$ that is active given $C$ in $\cG_1$.

Let $\cQ:a-\dots-b$ be an arbitrary common path shared by $\cG_1$ and $\cG_2$ such that it is active given $C$ in $\cG_2$. For example, $\cQ$ can be $\cP$. By the assumption, $\cQ$ is inactive given $C$ in $\cG_1$. Since $\cQ$ is active given $C$ in $\cG_2$, every non-collider $c\in \cQ$ satisfies $c\notin C$ and every collider $d\in\cQ$ satisfies $\cde_{\cG_2}(d)\cap C\neq \varnothing$. Since $\cQ$ is inactive given $C$ in $\cG_1$, there exists a collider $d\in\cQ$ such that $\cde_{\cG_1}(d)\cap C=\varnothing$. This is because all (if any) non-colliders on $\cQ$ in $\cG_1$ must not belong to $C$: since $\cG_1$ differs from $\cG_2$ by only one missing edge, any non-collider on $\cQ$ in $\cG_1$ is also a non-collider in $\cG_2$. By the fact that any non-collider on $\cQ$ in $\cG_2$ is not in $C$, there is no non-collider $c\in\cQ\cap C$ in $\cG_1$. Note that $d$ is also a collider in $\cG_2$. Therefore it must hold that $\cde_{\cG_2}(d)\cap C\neq \varnothing$, otherwise $\cQ$ is inactive given $C$ in $\cG_2$.

Following the arguments in the former paragraph, we know that there must exist a shared collider $d\in\cQ$ such that $\cde_{\cG_2}(d)\cap C\neq \cde_{\cG_1}(d)\cap C=\varnothing$. Since $\cde_{\cG_2}(d)\cap C\neq \varnothing$, there exists a directed path $d\rightarrow \cdots \rightarrow c\in C$ in $\cG_2$; denote $L_C(\cQ)$ as the length of the shortest directed path like this for a given $\cQ$.

Consider the following $\cQ$. Denote $[\cP]$ as the set of all common paths connecting $a,b$ shared by $\cG_1$ and $\cG_2$ that are active given $C$ in $\cG_2$. Since $\cP\in[\cP]$, we know $[\cP]\neq \varnothing$. Let $L$ be the length of the shortest path in $[\cP]$. Let $\cQ\in[\cP]$ be the path among all paths with length $L$ such that $L_C(\cQ)$ is minimized. We will show a contradiction.

Let the nodes $e,d,f,c$ be such that $e\rightarrow d\leftarrow f \in \cQ$, $\cde_{\cG_2}(d)\cap C\neq \cde_{\cG_1}(d)\cap C=\varnothing$, and there is a directed path $\cL: d\rightarrow \dots \rightarrow c\in C$ of length $L_C(\cQ)$ in $\cG_2$. Since $\cde_{\cG_1}(d)\cap C=\varnothing$, path $\cL$ must not be in $\cG_1$. Since $\cG_1$ only differs from $\cG_2$ by missing one edge $i\rightarrow j$, we know $i\rightarrow j$ must be on $\cL$. Suppose $d\rightarrow d_1\rightarrow\dots\rightarrow d_{k-1}\to i\to j$ on $\cL$ for some integer $k$. Denote for simplicity $d_{k}=i$ and $d_{k+1}=j$.

If $e\not\sim f$ in $\cG_2$, then $e\rightarrow d\leftarrow f$ construes a v-structure which is directed in $\cE(\cG_2)$. If there is $k'\in [k+1]$ such that $e - d_{k'} - f\in \cG_2$, then by acyclicity we have $e\rightarrow d_{k'}\leftarrow f\in \cG_2$. Then the path $\cQ_1: a-\dots-e\to d_{k'}\from f-\dots -b$ by removing $d$ from $\cQ$ and connecting $e$ to $f$ by $d_{k'}$ has length $L$ but a directed path from $d_{k'}$ to $c$ in $\cG_2$ of length $L_C(\cQ)-k'<L_C(\cQ)$. Note that $\cQ_1$ is also connecting $a,b$, shared by $\cG_1$ and $\cG_2$, and active given $C$ in $\cG_2$ (since it has the same set of non-colliders as $\cQ$ and the additional collider $d_{k'}$ has descendant $c\in C$ in $\cG_2$). This contradicts $\cQ$ having $L_C(\cQ)$ minimized. Therefore for all $k'\in [k+1]$, node $d_{k'}$ will not be adjacent to both $e$ and $f$ in $\cG_2$. By using Meek rule 1 (Proposition~\ref{prop:1}) recursively, we know that $d\rightarrow d_{1}\rightarrow\dots \rightarrow d_{k+1}$ is directed in $\cE(\cG_2)$. This contradicts $i\to j$ being undirected in $\cE(\cG_2)$.

\begin{figure}[h]
    \centering
     \begin{subfigure}{0.4\textwidth}
         \centering
         \includegraphics[width=0.38\textwidth]{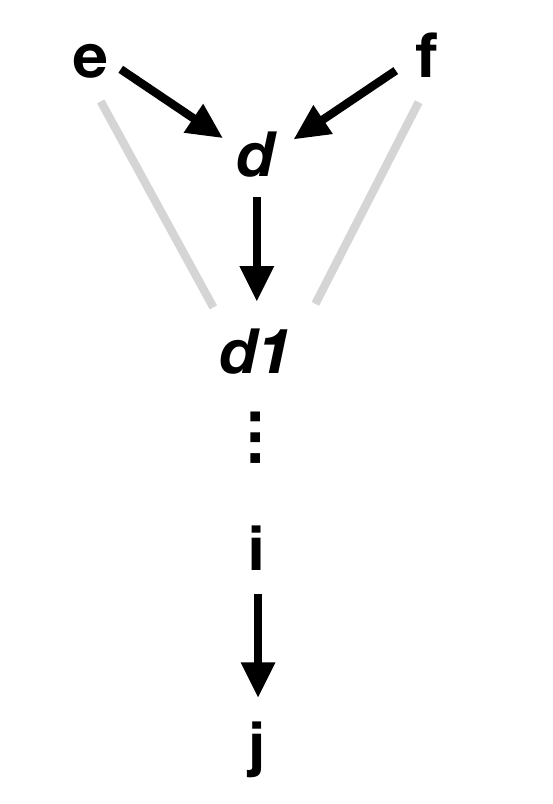}
         \caption{$e\not\sim f$ in $\cG_2$}
     \end{subfigure}
     \begin{subfigure}{0.45\textwidth}
         \centering
         \includegraphics[width=0.38\textwidth]{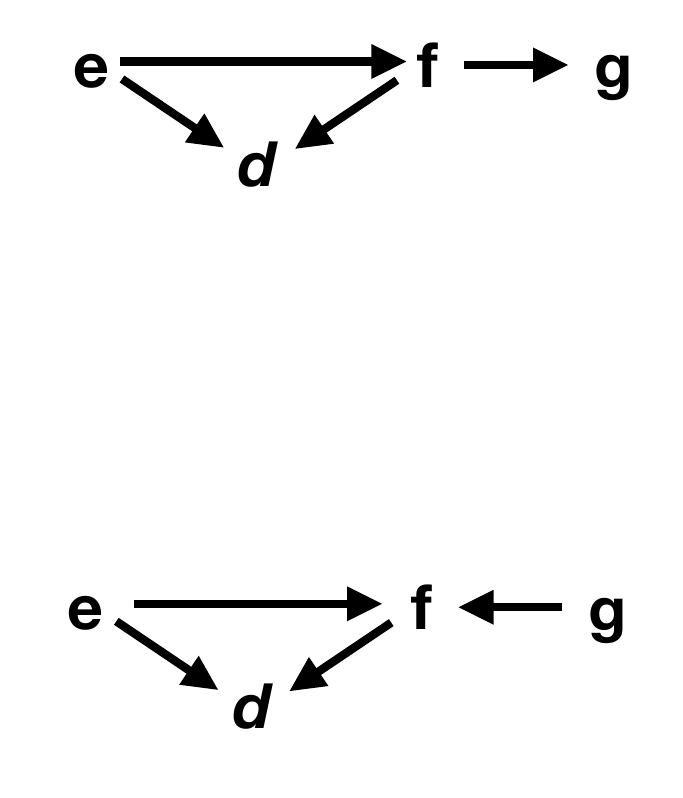}
         \caption{$e\sim f$ in $\cG_2$}
     \end{subfigure}     
\caption{Illustration for the proof of Lemma~\textcolor{red}{5}.}\label{fig:a3}
\end{figure}

If $e\sim f$ in $\cG_2$, we can assume without loss of generality that $e\to f\in\cG_2$. Denote $g$ as the node such that $e\to d\leftarrow f-g\in \cL'$. If $f\to g\in \cG_2$, then the path $\cQ_2: a-\dots -e\to f\to g-\dots -b$ by removing $d$ from $\cQ$ and connecting $e$ to $f$ by $e\to f$ has length $L-1<L$. However, it is also connecting $a,b$, shared by $\cG_1$ and $\cG_2$, and active given $C$ in $\cG_2$ (since it has the same set of non-colliders as $\cQ$ and its colliders are also colliders of $\cQ$). This contradicts $L$ being the smallest. If $f\leftarrow g\in \cG_2$, then the path $\cQ_3:a-\dots -e\to f\from g-\dots -b$ obtained similarly as $\cQ_2$ also shares similar properties as $\cQ_2$. It is active given $C$ in $\cG_2$ because its non-colliders are also non-colliders of $\cQ$ and the additional collider $f$ is a parent of $d$ which has descendant $c\in C$ in $\cG_2$. This contradicts $L$ being the smallest. Therefore there is always a contradiction if we assume the contrary of (2). Thus (2) is proven.
\end{proof}

An immediate corollary of this lemma is given in Corollary~\textcolor{red}{6}. We restate the corollary below. Note that a formal proof is provided in the main text.
\begin{namedlemma}[Corollary 6]\label{cor:2dag1var-c}
     Let $\cG_1$ and $\cG_2$ be two DAGs such that $\cG_1$ differs from $\cG_2$ \textup{only} by the missing edge $i\to j$, which is undirected in $\cE(\cG_2)$. For any nodes $a\neq b\in [n]$ and set $C\subset [n]\setminus\{a,b\}$, 
     \begin{enumerate}
         \item[(1)] if $a\ci b\mid_{\cG_2} C$, then $a\ci b\mid_{\cG_1} C$;
         \item[(2)] if $a\nci b\mid_{\cG_2} C$ and there is path from $a$ to $b$ in $\cG_1$ that is active given $C$ in $\cG_2$, then $a\nci b\mid_{\cG_1} C$. 
     \end{enumerate}
\end{namedlemma}

\subsection{Proof of Lemma~\textcolor{red}{7}}

Using Corollary~\textcolor{red}{6}, we can prove Lemma~\textcolor{red}{7} restated below.
\begin{namedlemma}[Lemma 7]\label{lm:2dag1ineq}
Let $\cG_1$ and $\cG_2$ be two DAGs such that $\cG_1$ differs from $\cG_2$ \textup{only} by the missing edge $i\to j$, which is undirected in $\cE(\cG_2)$. Denote the maximal undirected clique in $\cG_2$ containing $i,j$ by $\cS$. Then $I_{\cG_1}(A,B\mid C)\neq I_{\cG_2}(A,B\mid C)$ \textup{only} if $A\ci B\mid_{\cG_1} C$, $A\nci B\mid_{\cG_2} C$, and 
\[
\left(\pa_{\cG_2}(j)\cap\ch_{\cG_2}(i)\right) \cap \cS \subseteq C\cap \cS \subseteq \left(\pa_{\cG_2}(j)\setminus\{i\}\right) \cap \cS.\footnote{In fact, one can show $C\cap \cS =\left(\pa_{\cG_2}(j)\setminus\{i\}\right) \cap \cS$. However, we will only make use of this lemma.}
\]
\end{namedlemma}

\begin{proof}
By Corollary~\textcolor{red}{6} (1), we know that $A\ci B\mid_{\cG_2} C$ would mean $A\ci B\mid_{\cG_1} C$. Thus $I_{\cG_1}(A,B\mid C)\neq I_{\cG_2}(A,B\mid C)$ \textup{only} if $A\ci B\mid_{\cG_1} C$, $A\nci B\mid_{\cG_2} C$. Then by Corollary~\textcolor{red}{6} (2), we know that $A\ci B\mid_{\cG_1} C$, $A\nci B\mid_{\cG_2} C$ only if every path from $A$ to $B$ in $\cG_1$ is inactive given $C$ in $\cG_2$. Since $A\nci B\mid_{\cG_2} C$, there must be a path $\cP$ from $a\in A$ to $b\in B$ that is active given $C$ in $\cG_2$. Since $\cG_1$ differs from $\cG_2$ by only missing one edge $i\to j$ and $\cP$ is not a path in $\cG_1$, then $\cP$ must contain $i\to j$ on it. Suppose $\cP:a-\dots -i\to j -\dots - b$. 

\begin{figure}[h]
% \vspace{.3in}
% \vspace{.3in}
\centering
\includegraphics[width=0.6\textwidth]{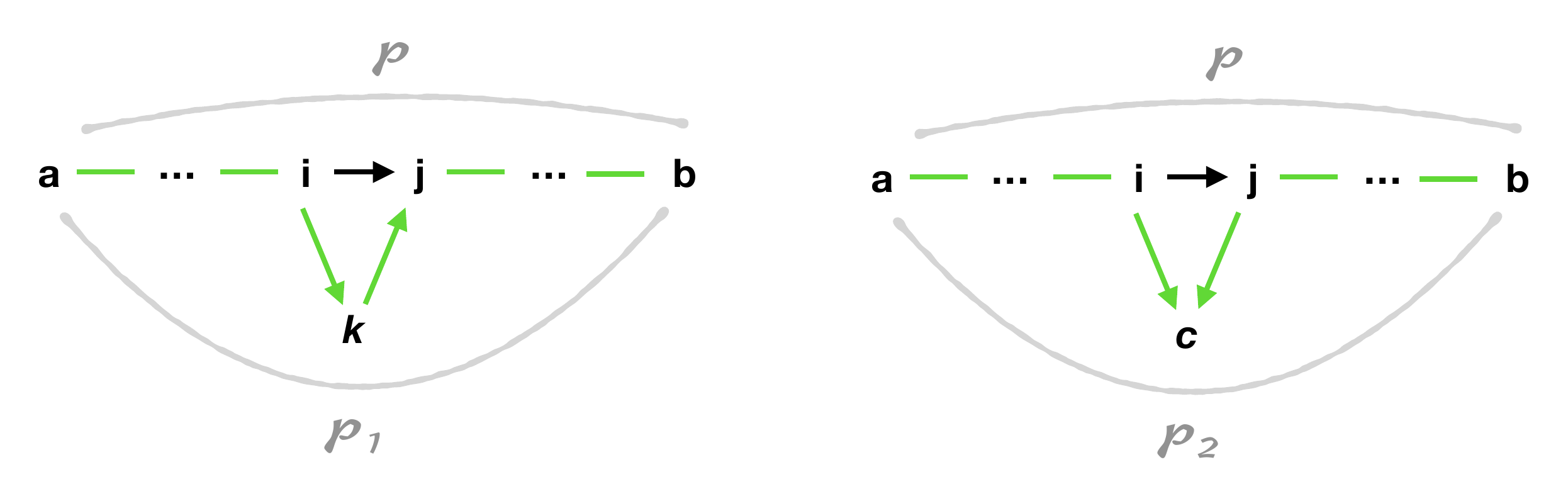}
% \vspace{.3in}
\caption{Illustration of paths $\cP$, $\cP_1,\cP_2$.}\label{fig:a4}
\vspace{-.1in}
\end{figure}

We first show that $\left(\pa_{\cG_2}(j)\cap\ch_{\cG_2}(i)\right) \cap \cS \subseteq C\cap \cS $. Let $k$ be an arbitrary node in $\left(\pa_{\cG_2}(j)\cap\ch_{\cG_2}(i)\right) \cap \cS$. If $k\notin C$, then the path $\cP_1:
a-\dots -i\to k\to j-\dots b$ by removing $i\to j$ from $\cP$ but connecting $i,j$ using $k$ is active given $C$ in $\cG_2$. This is because all its colliders are colliders of $\cP$ and the additional non-collider $k\notin C$. However, $\cP_1$ is in $\cG_1$ as it does not contain $i\to j$; a contradiction. Therefore $k\in C$ and $\left(\pa_{\cG_2}(j)\cap\ch_{\cG_2}(i)\right) \cap \cS \subseteq C\cap \cS$.

We now show that $C\cap \cS \subseteq \left(\pa_{\cG_2}(j)\setminus\{i\}\right) \cap \cS$. Suppose there is $c\in C\cap \cS$ such that $c\notin\pa_{\cG_2}(j)\setminus\{i\}$. Since $c\in\cS$, it is adjacent to both $i,j$ in $\cG_2$. Thus $i\rightarrow c\leftarrow j$ by $c\notin\pa_{\cG_2}(j)\setminus\{i\}$ and $i\in\pa_{\cG_2}(j)$. Then the path $\cP_2:a-\dots -i\to c\from j-\dots -b$ by removing $i\to j$ from $\cP$ and connecting $i,j$ using $i\rightarrow c\leftarrow j$ is active given $C$ in $\cG_2$ (as all its non-colliders are non-colliders of $\cP$ and the additional collider $c\in C$). However, $\cP_2$ is in $\cG_1$ as it does not contain $i\to j$; a contradiction. Therefore $C\cap \cS \subseteq \left(\pa_{\cG_2}(j)\setminus\{i\}\right) \cap \cS$.
\end{proof}

With these results the lower bound in Theorem~\textcolor{red}{1} can be proven as described in Section~\textcolor{red}{4} in the main text.

% \kk{Very nice argument at the end of the above paragraph. Thinking a little more, I think we can sum over the size of $\cK$ and get the following stronger result. What do you think?}

% \kk{Can we write upper bound in the form of $2^{m(\cG') + O(\log n)}$? Note that the coefficient of $m(\cG')$ is 1. If we argue that we can test using lets say $O(n^3) m(G)!$ number of independence tests it is not that nice, because $m(G)!$ is like $e^{m(G) \log m(G)}$. We should try our best to achieve $2^{m(G) + O(\log n)}$ result for the number of independence tests needed to do the testing.}

% \kk{Another inequality which could be useful: $\frac{2^{2n}}{\sqrt{\pi n}} (1-1/8n) \leq \binom{2n}{n} \leq \frac{2^{2n}}{\sqrt{\pi n}} (1-1/9n)$. Therefore the following improved result is not that important.}

% \kk{As all topological orders can be valid, it can take $\sum_{|\cK|=1}^{m(\cG')-2} \binom{m(\cG')-2}{|\cK|}=2^{m(\cG')-2}$ conditional independence test until $C\cap \cS= \cK$ in the worst case.}

% \jj{I think the sum over $|K|$ only improve from $\binom{s}{s/2}$ to $2^{s-2}$, which are both $\exp(\Omega(s))$. It doesn't really bring in $n$ (which is the number of total nodes). So maybe we can skip it. We shouldn't be able to get at $\exp(\Omega(s)+\Omega(\log n))$ - consider a very empty graph with just a tiny part being a clique.}

\section{OMITTED PROOFS OF UPPER BOUND}

We first explain why the total number of
maximal undirected cliques in $[\cG]$ cannot exceed $n$. Since the undirected edges in $\cE(\cG)$ correspond to a collection of chordal chain components, it suffices to show that in any chordal graph of size $k$, the number of maximal undirected cliques is at most $k$. This is a well-known result; see for example \citep{dirac1961rigid}.

To show the upper bound in Theorem~\textcolor{red}{2}, we only need to prove Proposition \textcolor{red}{9} and Lemma \textcolor{red}{10}. With these results, Theorem~\textcolor{red}{2} can be obtained as described in Section~\textcolor{red}{5.2} in the main text.

\begin{namedlemma}[Proposition 9]
    If $\cH\leq\cG$ and $\skel(\cH)\subsetneq \skel(\cG)$, then there exists a DAG $\cH'$ such that $\cH\leq\cH'\leq\cG$ and $\cH'$ is missing one edge in $\cG'$ for some $\cG'\in [\cG]$.
\end{namedlemma}

\begin{proof}
    By Theorem 4 in \citep{chickering2002optimal}, there exists a sequence of DAGs $\cG_1,\dots,\cG_k$ such that 
    \begin{itemize}
        \item $\cH=\cG_{k+1}\leq \cG_k\leq\dots \leq \cG_1\leq \cG_0=\cG$;
        \item For each $i\in\{0,\dots,k+1\}$, $\cG_{i+1}$ differs from $\cG_i$ by either a covered edge flip in $\cG_i$ or a missing edge from $\cG_i$.
    \end{itemize}

    Note that by Lemma 2 in \citep{chickering2002optimal}, if $\cG_{i+1}$ differs from $\cG_i$ by a covered edge flip, then they are in the same MEC. Let $i$ be the smallest index such that $\cG_{i+1}$ is missing an edge from $\cG_i$. Such an $i$ exists since $\skel(\cH)\neq \skel(\cG)$. Then $[\cG_{i}]=\cdots=[\cG_0]=[\cG]$. Therefore letting $\cG'=\cG_{i}$ and $\cH=\cG_{i+1}$ concludes the proof.
\end{proof}

\begin{namedlemma}[Lemma 10]
    For any $i$ in $\cG'\in[\cG]$, there is $\pa_{\cG'}(i)=\pa_{[\cG]}(i)\cup C$ for some undirected clique $C\subseteq\adj_{[\cG]}(i)$.
\end{namedlemma}

\begin{proof}
    As reviewed in Section~\textcolor{red}{2.3}, all directed edges in the essential graph of $[\cG]$ are shared by $\cG'$. Therefore $\pa_{\cG'}(i)\subseteq \pa_{[\cG]}\cup\adj_{[\cG]}(i)$. 
    
    We now show that $\pa_{\cG'}(i)\cap \adj_{[\cG]}(i)$ is an undirected clique. For $j\neq k\in \pa_{\cG'}(i)\cap \adj_{[\cG]}(i)$, if $j\not\sim k$ in $[\cG]$, then $j\to i\from k$ construes a v-structure in $\cG'$ that is not in $[\cG]$; a contradiction to $\cG'\in[\cG]$. Thus $j\sim k$ in $[\cG]$. Furthermore $j\sim k$ must be undirected. Otherwise assume $j\to k$ without loss of generality. We have $j\sim i\sim k$ in one undirected chain component of $[\cG]$ and $j\sim i$, $i\sim k$ in different maximal cliques. As reviewed in Section~\textcolor{red}{2.3}, it follows from \citet{wienobst2021polynomial} that there is some DAG in $[\cG]$ such that the maximal clique containing $i\sim k$ is the most upstream and $k\to i$. This creates a cycle $k\to i\to j\to k$; a contradiction. Therefore $j\sim k$ must be undirected. Thus $\pa_{\cG'}(i)\cap \adj_{[\cG]}(i)$ is an undirected clique. 
\end{proof}

\section{AN EXAMPLE}

We give an example of a DAG whose maximum undirected clique size is $2$ but its MEC has size $\exp(\Omega(n))$. Consider the DAG in Figure~\ref{fig:a5}: it starts off with two v-structures and then continues with repeated block structures. In its essential graph, all black edges are directed (from v-structures and Meek rule 1), whereas all green edges are undirected and can be oriented in all possible ways. Therefore its maximum undirected clique size is $2$ but its MEC has size $2^{n/2-1}=\exp(\Omega(n))$.

\begin{figure}[h]
% \vspace{.3in}
% \vspace{.3in}
\centering
\includegraphics[width=0.4\textwidth]{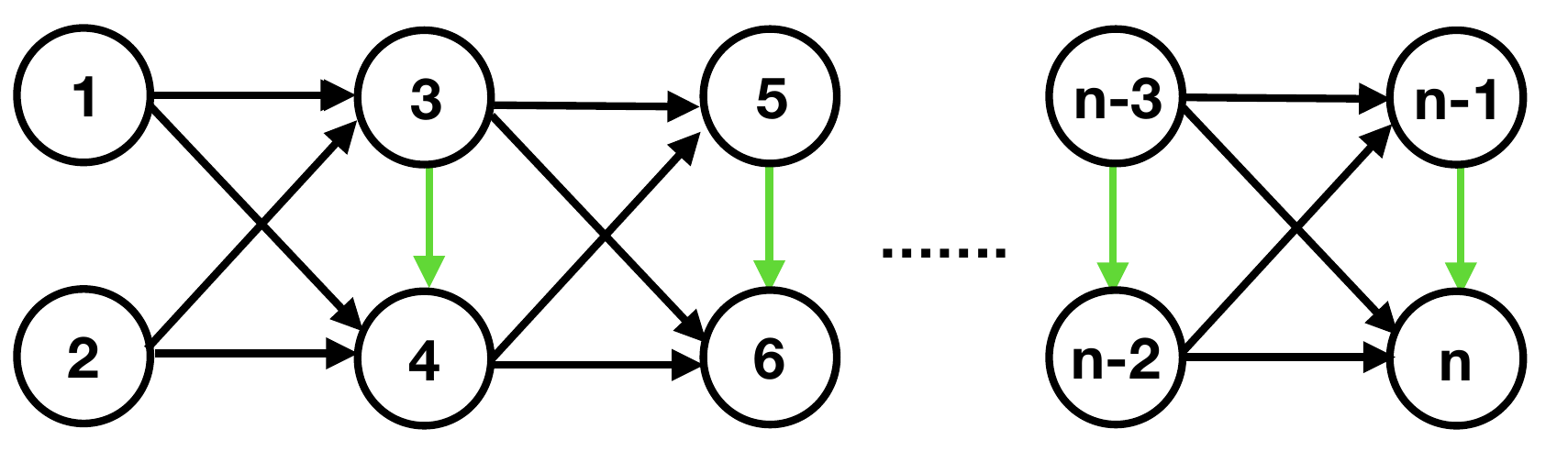}
% \vspace{.3in}
\caption{An example DAG.}\label{fig:a5}
\vspace{-.1in}
\end{figure}

% \newpage
% \bibliographystyleappendix{alpha}
% \bibliographyappendix{refs}

\end{document}